\documentclass{article}
\usepackage{amsfonts}

\newtheorem{theorem}{Theorem}

\newtheorem{lemma}[theorem]{Lemma}
\newtheorem{notation}[theorem]{Notation}

\newenvironment{proof}[1][Proof]{\noindent\textbf{#1.} }{\ \rule{0.5em}{0.5em}}
\input{tcilatex}

\begin{document}

\title{An optimization problem on the sphere}
\author{Andreas Maurer \\
Adalbertstr 55\\
D80799 M\"{u}nchen}
\maketitle

\begin{abstract}
We prove existence and uniqueness of the minimizer for the average geodesic
distance to the points of a geodesically convex set on the sphere. This
implies a corresponding existence and uniqueness result for an optimal
algorithm for halfspace learning, when data and target functions are drawn
from the uniform distribution.
\end{abstract}

\section{Introduction}

Let $\mathcal{S}^{n-1}$ be the unit sphere in $%
\mathbb{R}
^{n}$ with normalized uniform measure $\sigma $ and geodesic metric $\rho $,
and let $K$ be a proper convex cone with nonempty interior in $%
\mathbb{R}
^{n}$. We will show that the function $\psi :\mathcal{S}^{n-1}\rightarrow 
\mathbb{R}
$ defined by%
\[
\psi _{K}\left( w\right) =\int_{K\cap \mathcal{S}^{n-1}}\rho \left(
w,y\right) d\sigma \left( y\right) 
\]%
attains its global minimum at a unique point on $\mathcal{S}^{n-1}$. While
existence of the minimum is straightforward, uniqueness seems surprisingly
difficult to prove.

A similar problem has been considered in \cite{Galperin 1993} and \cite{Buss
2001}. In these works the intention is to define a centroid, so integration
is replaced by finite summation and $\rho \left( w,y\right) $ replaced by $%
\rho \left( w,y\right) ^{2}$. Since the problem is rather obvious, it
appears likely that a proof of the above result exists somewhere in the
literature and we just haven't been able to find it.

\section{Optimal halfspace learning}

Our motivation to consider this problem arises in learning theory.
Specifically we consider an experiment, where

\begin{enumerate}
\item A unit vector $u$ is drawn at random from $\sigma $ and kept concealed
from the learner.

\item A sample $\mathbf{x}=\left( x_{1},...,x_{m}\right) \in \left( \mathcal{%
S}^{n-1}\right) ^{m}$ is generated in $m$ independent random trials of $%
\sigma $.

\item A label vector $\mathbf{y}=u\left( \mathbf{x}\right) \in \left\{
-1,1\right\} ^{m}$ is generated according to the rule $y_{i}=sign\left(
\left\langle u,x_{i}\right\rangle \right) $, where $\left\langle
.,.\right\rangle $ is the euclidean inner product and $sign\left( t\right)
=1 $ if $t>0$ and $-1$ if $t<0$. The $sign$ of $0$ is irrelevant, because it
corresponds to events of probability zero.

\item The labeled sample $\left( \mathbf{x},\mathbf{y}\right) =\left( 
\mathbf{x},u\left( \mathbf{x}\right) \right) $ is supplied to the learner.

\item The learner produces a hypothesis $f\left( \mathbf{x},\mathbf{y}%
\right) \in \mathcal{S}^{n-1}$ according to some learning rule $f:\left( 
\mathcal{S}^{n-1}\right) ^{m}\times \left\{ -1,1\right\} ^{m}\rightarrow 
\mathcal{S}^{n-1}$.

\item An unlabeled test point $x\in \mathcal{S}^{n-1}$ is drawn at random
from $\sigma $ and presented to the learner who produces the label $%
y=sign\left( \left\langle f\left( \mathbf{x},\mathbf{y}\right)
,x\right\rangle \right) $.

\item If $sign\left( \left\langle u,x_{i}\right\rangle \right) =y$ the
learner is rewarded one unit, otherwise a penalty of one unit is incurred.
\end{enumerate}

We now ask the following question: Which learning rule $f$ will give the
highest average reward on a very large number of independent repetitions of
this experiment?

Evidently the optimal learning rule has to minimize the following functional:%
\[
\Omega \left( f\right) =\mathbb{E}_{u\sim \sigma }\mathbb{E}_{\mathbf{x}\sim
\sigma ^{m}}\Pr_{x\sim \sigma }\left\{ sign\left( \left\langle f\left( 
\mathbf{x},u\left( \mathbf{x}\right) \right) ,x\right\rangle \right) \neq
sign\left( \left\langle u,x\right\rangle \right) \right\} . 
\]%
Now a simple geometric argument shows that for any $v,u\in \mathcal{S}^{n-1}$
we have%
\[
\Pr_{x\sim \sigma }\left\{ sign\left( \left\langle v,x\right\rangle \right)
\neq sign\left( \left\langle u,x\right\rangle \right) \right\} =\rho \left(
v,u\right) /\pi , 
\]%
relating the misclassification probability to the geodesic distance. For a
labeled sample $\left( \mathbf{x},\mathbf{y}\right) \in \left( \mathcal{S}%
^{n-1}\right) ^{m}\times \left\{ -1,1\right\} ^{m}$ we denote 
\[
C\left( \mathbf{x},\mathbf{y}\right) =\left\{ u\in \mathcal{S}^{n-1}:u\left( 
\mathbf{x}\right) =\mathbf{y}\right\} \text{.} 
\]%
$C\left( \mathbf{x},\mathbf{y}\right) $ is thus the set of all hypotheses
consistent with the labeled sample $\left( \mathbf{x},\mathbf{y}\right) $.
Observe that, given $\mathbf{x}$ and $u$ there is exactly one $\mathbf{y}$
such that $\mathbf{y}=u\left( \mathbf{x}\right) $, that is $u\in C\left( 
\mathbf{x},\mathbf{y}\right) $. We also have $C\left( \mathbf{x},\mathbf{y}%
\right) =K\left( \mathbf{x},\mathbf{y}\right) \cap \mathcal{S}^{n-1}$ where $%
K\left( \mathbf{x},\mathbf{y}\right) $ is the closed convex cone%
\[
K\left( \mathbf{x},\mathbf{y}\right) =\left\{ v\in 
\mathbb{R}
^{n}:\left\langle u,y_{i}x_{i}\right\rangle \geq 0,\forall \,1\leq i\leq
m\right\} . 
\]%
We therefore obtain%
\begin{eqnarray*}
\Omega \left( f\right) &=&\pi ^{-1}\mathbb{E}_{u\sim \sigma }\mathbb{E}_{%
\mathbf{x}\sim \sigma ^{m}}\rho \left( f\left( \mathbf{x},u\left( \mathbf{x}%
\right) \right) ,u\right) \\
&=&\pi ^{-1}\mathbb{E}_{\mathbf{x}\sim \sigma ^{m}}\sum_{\mathbf{y}\in
\left\{ -1,1\right\} ^{m}}\mathbb{E}_{u\sim \sigma }\rho \left( f\left( 
\mathbf{x},u\left( \mathbf{x}\right) \right) ,u\right) 1_{C\left( \mathbf{x},%
\mathbf{y}\right) }\left( u\right) \\
&=&\pi ^{-1}\mathbb{E}_{\mathbf{x}\sim \sigma ^{m}}\sum_{\mathbf{y}\in
\left\{ -1,1\right\} ^{m}}\mathbb{E}_{u\sim \sigma }\rho \left( f\left( 
\mathbf{x},\mathbf{y}\right) ,u\right) 1_{C\left( \mathbf{x},\mathbf{y}%
\right) }\left( u\right) \\
&=&\pi ^{-1}\mathbb{E}_{\mathbf{x}\sim \sigma ^{m}}\sum_{\mathbf{y}\in
\left\{ -1,1\right\} ^{m}}\psi _{K\left( \mathbf{x},\mathbf{y}\right)
}\left( f\left( \mathbf{x},\mathbf{y}\right) \right) .
\end{eqnarray*}%
If $K\left( \mathbf{x},\mathbf{y}\right) $ has empty interior then the
corresponding summand vanishes, so we can assume that $K\left( \mathbf{x},%
\mathbf{y}\right) $ has nonempty interior. Clearly $-y_{i}x_{i}\notin
K\left( \mathbf{x},\mathbf{y}\right) $ for all example points, so $K\left( 
\mathbf{x},\mathbf{y}\right) $ is a proper cone. Our result therefore
applies and asserts the existence of a unique minimizer $f^{\ast }\left( 
\mathbf{x},\mathbf{y}\right) $ of the function $\psi _{K\left( \mathbf{x},%
\mathbf{y}\right) }$. The map $f^{\ast }:\left( \mathbf{x},\mathbf{y}\right)
\mapsto f^{\ast }\left( \mathbf{x},\mathbf{y}\right) $ is then the unique
optimal learning algorithm.

The map $f^{\ast }$ also has the symmetry property $f^{\ast }\left( V\mathbf{%
x},\mathbf{y}\right) =Vf^{\ast }\left( \mathbf{x},\mathbf{y}\right) $ for
any unitary $V$ on $%
\mathbb{R}
^{n}$. This is so, because 
\[
\psi _{K\left( V\mathbf{x},\mathbf{y}\right) }\left( w\right) =\psi
_{K\left( \mathbf{x},\mathbf{y}\right) }\left( V^{-1}w\right) , 
\]%
as is easily verified. We will also show, that the solution $f^{\ast }\left( 
\mathbf{x},\mathbf{y}\right) $ must lie in the cone%
\[
\left\{ \sum_{i=1}^{m}\alpha _{i}y_{i}x_{i}:\alpha _{i}\geq 0\right\} 
\]%
and that $\psi _{K\left( \mathbf{x},\mathbf{y}\right) }$ has no other local
minima.

\section{Proof of the main result}

\begin{notation}
$\rho \left( .,.\right) $ is the geodesic distance and $\sigma $ the Haar
measure on $\mathcal{S}^{n-1}$. For $A\subseteq 
\mathbb{R}
^{n}$ we denote $A_{1}=\left\{ x\in A:\left\Vert x\right\Vert =1\right\}
=A\cap \mathcal{S}^{n-1}$. 'Cone' will always mean 'convex cone'. For $%
A\subseteq 
\mathbb{R}
^{n}$ we denote%
\[
\hat{A}=\left\{ x:\left\langle x,v\right\rangle \geq 0,\forall v\in
A\right\} . 
\]%
This is always a closed convex set. A proper cone $K$ is one which is
contained in some closed halfspace. For a set $A$ the indicator function of $%
A$ will be denoted by $1_{A}$\bigskip
\end{notation}

\begin{lemma}
Let $K$ be a closed cone

(i) If $w\notin K$ then there is a unit vector $z\in 
\mathbb{R}
^{n}$ such that $\left\langle z,w\right\rangle <0$ and $\left\langle
z,y\right\rangle \geq 0$ for all $y\in K$.

(ii) $\left( \hat{K}\right) ^{\symbol{94}}=K\,$.

(iii) Suppose that $K$ is proper and has nonempty interior, $w\in \mathcal{S}%
^{n-1}$, $w\notin \hat{K}\cup \left( -\hat{K}\right) $ and $\epsilon >0$.
Then there exists $z$ with $\left\Vert z\right\Vert =1$ such that $-\epsilon
<\left\langle z,w\right\rangle <0$ and $\left\langle z,y\right\rangle >0$
for all $y\in \hat{K}\backslash \left\{ 0\right\} $.
\end{lemma}

\begin{proof}
(i) Let $B$ be an open ball containing $w$ such that $K\cap B=\emptyset $.
Define 
\[
O=\left\{ \lambda x:x\in B,\lambda >0\right\} \text{.} 
\]%
Then $K$ and $O$ are nonempty disjoint convex sets and $O$ is open. By the
Hahn-Banach theorem (\cite{Rudin 1974}, Theorem 3.4) there is $\gamma \in 
\mathbb{R}
$ and $z\in 
\mathbb{R}
^{n}$ such that 
\[
\left\langle z,x\right\rangle <\gamma \leq \left\langle z,y\right\rangle
,\forall x\in O,y\in K. 
\]%
Choosing $y=0\in K$ gives $\gamma \leq 0$, letting $\lambda \rightarrow 0$
in $\left\langle z,\lambda w\right\rangle <\gamma $ shows $\gamma \geq 0$,
so that $\gamma =0$. The normalization is trivial.

(ii) Trivially $K\subseteq \left( \hat{K}\right) ^{\symbol{94}}$. On the
other hand, if $w\notin K$ let $z$ be the vector from part (i). Then $z\in 
\hat{K}$ but $\left\langle w,z\right\rangle <0$, so that $w\notin \left( 
\hat{K}\right) ^{\symbol{94}}$.

(iii) Since $w\notin \hat{K}$ there exists $x_{1}\in K$ s.t. $\left\langle
w,x_{1}\right\rangle <0$. Since the interior of $K$ is nonempty, $K$ is the
closure of its interior (Theorem 6.3 in \cite{Rockafellar 1970}), so we can
assume $x_{1}\in $ int$\left( K\right) $. Similarily, since $w\notin \left( -%
\hat{K}\right) $ we have $-w\notin \hat{K}$, so there is $x_{2}\in $ int$%
\left( K\right) $ with $\left\langle -w,x_{2}\right\rangle <0$, that is $%
\left\langle w,x_{2}\right\rangle >0$. Since the interior of $K$ is convex
it contains the segment $\left[ x_{1},x_{2}\right] $, so by continuity of $%
\left\langle w,\cdot \right\rangle $ there is some $x_{0}\in $ int$\left(
K\right) $ with $\left\langle w,x_{0}\right\rangle =0$. Since $K$ is a
proper cone $0\notin $ int$\left( K\right) $ and we can assume that $%
\left\Vert x_{0}\right\Vert =1$.

Let $c>0$ be such that $x^{\prime }\in K$ whenever $\left\Vert
x_{0}-x^{\prime }\right\Vert \leq c$. We define 
\[
z=\left( 1-\eta \right) ^{1/2}x_{0}-\eta ^{1/2}w\text{, where }0<\eta <\min
\left\{ \frac{c^{2}}{1+c^{2}},\epsilon ^{2}\right\} \text{.} 
\]%
Since $\left\langle w,x_{0}\right\rangle =0$ it is clear that $z$ is a unit
vector. Also $\left\langle w,z\right\rangle =-\eta ^{1/2}>-\epsilon $, and
for any $y\in \hat{K}_{1}$ we have $x_{0}-cy\in K$, so $\left\langle
y,x_{0}-cy\right\rangle \geq 0$ and%
\begin{eqnarray*}
\left\langle y,z\right\rangle &=&\left( 1-\eta \right) ^{1/2}\left(
\left\langle y,x_{0}-cy\right\rangle +c\left\langle y,y\right\rangle \right)
-\eta ^{1/2}\left\langle y,w\right\rangle \\
&\geq &\left( 1-\eta \right) ^{1/2}c-\eta ^{1/2}>0.
\end{eqnarray*}
\end{proof}

\begin{theorem}
Let $K\subset 
\mathbb{R}
^{n-1}$ be a closed proper cone with nonempty interior, $g:\left[ 0,\pi %
\right] \rightarrow 
\mathbb{R}
$ continuous and the function $\psi :\mathcal{S}^{n-1}\rightarrow 
\mathbb{R}
$ defined by%
\[
\psi \left( w\right) =\int_{K_{1}}g\left( \rho \left( w,y\right) \right)
d\sigma \left( y\right) . 
\]

(i) $\psi $ attains its global minimum on $\mathcal{S}^{n-1}$.

(ii) If $g$ is increasing then every local minimum of $\psi $ must lie in $%
\hat{K}\cup \left( -\hat{K}\right) $ and every global minimum of $\psi $
must lie in $K\cap \hat{K}$.

(iii) If $g$ is increasing and convex in $\left[ 0,\pi /2\right] $ then the
global minimum of $\psi $ is unique and corresponds to the only local
minimum outside $-\hat{K}$.

(iv) If $g$ is increasing, convex in $\left[ 0,\pi /2\right] $ and concave
in $\left[ \pi /2,\pi \right] $ then the global minimum of $\psi $ is unique
and corresponds to its only local minimum on $\mathcal{S}^{n-1}$.
\end{theorem}

\begin{proof}
(i) is immediate since $\mathcal{S}^{n-1}$ is compact and $\psi $ is
continuous.

(ii) Fix $w\in \mathcal{S}^{n-1}$, $w\notin \hat{K}\cup \left( -\hat{K}%
\right) $. We will first show that there can be no local minimum of $\psi $
at $w$. Let $\epsilon >0$ be arbitrary and choose $z$ as in the lemma (iii).
The functional $z$ divides the sphere $\mathcal{S}^{n-1}$ into two open
hemispheres%
\[
L=\left\{ u:\left\langle z,u\right\rangle <0\right\} \text{ and }R=\left\{
u:\left\langle z,u\right\rangle >0\right\} , 
\]%
and an equator of $\sigma $-measure zero. Note that $w\in L$ and $\hat{K}%
_{1}\subseteq R$. We can write 
\[
c=\min_{y\in \hat{K}_{1}}\left\langle y,z\right\rangle >0, 
\]%
since $\hat{K}_{1}$ is compact and $y\mapsto \left\langle y,z\right\rangle $
is continuous. With $V$ we denote the reflection operator which exchanges
points of $L$ and $R$%
\[
Vx=-\left\langle x,z\right\rangle z+\left( x-\left\langle x,z\right\rangle
z\right) . 
\]%
$V$ is easily verified to an isometry and $V^{2}=I$.

Suppose now that $u\in R$ and $Vu\in K$. We claim that $u$ is in the
interior of $K$. Indeed, if $u^{\prime }\in 
\mathbb{R}
^{n}$ satisfies $\left\Vert u-u^{\prime }\right\Vert <2\left\langle
u,z\right\rangle c$, then for all $y\in \hat{K}_{1}$ we have 
\begin{eqnarray*}
\left\langle u^{\prime },y\right\rangle &=&\left\langle u,y\right\rangle
-\left\langle u-u^{\prime },y\right\rangle \geq \left\langle
u,y\right\rangle -2\left\langle u,z\right\rangle c \\
&\geq &\left\langle u,y\right\rangle -2\left\langle u,z\right\rangle
\left\langle z,y\right\rangle =\left\langle Vu,y\right\rangle \geq 0,
\end{eqnarray*}%
so $u^{\prime }\in \left( \hat{K}\right) ^{\symbol{94}}=K$, by part (ii) of
the lemma. This establishes the claim and shows that $V\left( K\right) \cap
R $ is contained in the interior of $K$. It follows that%
\begin{equation}
\forall u\in R,1_{K}\left( u\right) \geq 1_{K}\left( Vu\right) .
\label{reflection inequality}
\end{equation}%
Also $V\left( K\right) \cap R$ is relatively closed in $R$ while int$\left(
K\right) \cap R$ is open in $R$. Since $R$ is connected they can only
coincide if $V\left( K\right) \cap R=R$. But this is impossible, since then 
\begin{eqnarray*}
L\cup R &=&V\left( V\left( K\right) \cap R\right) \cup \left( V\left(
K\right) \cap R\right) \subseteq V\left( V\left( K\cap L\right) \right) \cup 
\text{int}\left( K\right) \\
&=&\left( K\cap L\right) \cup \text{int}\left( K\right) \subseteq K,
\end{eqnarray*}%
and $K$ is assumed to be a proper cone. So $V\left( K\right) \cap R$ is a
proper subset of int$\left( K\right) \cap R$. The inequality (\ref%
{reflection inequality}) is therefore strict on the nonempty open set $%
\left( \text{int}\left( K\right) \cap R\right) \backslash \left( V\left(
K\right) \cap R\right) $.

Using isometry and unipotence of $V$ we now obtain%
\begin{eqnarray*}
\psi \left( w\right) -\psi \left( Vw\right) &=&\int_{R}\left( g\left( \rho
\left( w,u\right) \right) -g\left( \rho \left( Vw,u\right) \right) \right)
1_{K}\left( u\right) d\sigma \left( u\right) + \\
&&+\int_{L}\left( g\left( \rho \left( w,u\right) \right) -g\left( \rho
\left( Vw,u\right) \right) \right) 1_{K}\left( u\right) d\sigma \left(
u\right) \\
&=&\int_{R}\left( g\left( \rho \left( w,u\right) \right) -g\left( \rho
\left( Vw,u\right) \right) \right) \left( 1_{K}\left( u\right) -1_{K}\left(
Vu\right) \right) d\sigma \left( u\right) \\
&>&0.
\end{eqnarray*}%
The inequality holds, because the first factor $\left( g\left( \rho \left(
w,u\right) \right) -g\left( \rho \left( Vw,u\right) \right) \right) $ in the
last integral is always positive for $u\in R$, since $g$ is increasing and $%
\rho $ is increasing in the euclidean distance. The second is nonnegative
and positive on a set of positive measure. Since $\left\Vert w-Vw\right\Vert
=2\epsilon $ and $\epsilon >0$ was arbitrary, we see that every neighborhood
of $w$ contains a point where $\psi $ has a smaller value than at $w$. We
conclude that $w$ cannot be a local minimum of $\psi $, which proves the
first assertion of (ii).

If $w\notin K$ choose $z$ as in part (i) of the lemma and let $W$ be the
isometry $Wx=-\left\langle x,z\right\rangle z+\left( x-\left\langle
x,z\right\rangle z\right) $. The $\forall u\in K$ we have $g\left( \rho
\left( w,u\right) \right) >g\left( \rho \left( Ww,u\right) \right) $, so $%
\psi \left( w\right) >\psi \left( Ww\right) $ and $w$ cannot be a global
minimizer of $\psi $. So every global minimizer must be in $K\cap \left( 
\hat{K}\cap \left( -\hat{K}\right) \right) $. Since $K_{1}\cap \left( -\hat{K%
}_{1}\right) $ is obviously empty the second assertion of (ii) follows.

(iii) Now let $w_{1},w_{2}\in \hat{K}_{1}$ with $w_{1}\neq w_{2}$. Connect
them with a geodesic in $\hat{K}_{1}$ and let $w^{\ast }\in \hat{K}_{1}$ be
the midpoint of this geodesic, such that $\rho \left( w_{1},w^{\ast }\right)
=\rho \left( w^{\ast },w_{2}\right) =\rho \left( w_{1},w_{2}\right) /2\leq
\pi /2$. We define a map $U$ by%
\[
Ux=\left\langle x,w^{\ast }\right\rangle w^{\ast }-\left( x-\left\langle
x,w^{\ast }\right\rangle w^{\ast }\right) . 
\]%
Geometrically $U$ is reflection on the one-dimensional subspace generated by 
$w^{\ast }$. Note that $w_{2}=Uw_{1}$ and that $\rho \left( u,Uu\right)
=2\rho \left( u,w^{\ast }\right) $ if $\rho \left( u,w^{\ast }\right) \leq
\pi /2$ and that $\rho \left( u,Uu\right) =2\pi -2\rho \left( u,w^{\ast
}\right) $ if $\rho \left( u,w^{\ast }\right) \geq \pi /2$.

Take any $u\in K_{1}$. Since $w^{\ast }\in \hat{K}_{1}$ we have $\rho \left(
u,w^{\ast }\right) \leq 2\pi $, whence $\rho \left( u,Uu\right) =2\rho
\left( u,w^{\ast }\right) $. All the four points $w_{1},w_{2},u$ and $Uu$
have at most distance $\pi /2$ from $w^{\ast }$ and lie therefore together
with $w^{\ast }$ on a common hemisphere. By the triangle inequality%
\begin{eqnarray*}
2\rho \left( u,w^{\ast }\right) &=&\rho \left( u,Uu\right) \\
&\leq &\rho \left( u,w_{1}\right) +\rho \left( w_{1},Uu\right) =\rho \left(
u,w_{1}\right) +\rho \left( Uw_{1},UUu\right) \\
&=&\rho \left( u,w_{1}\right) +\rho \left( w_{2},u\right) .
\end{eqnarray*}%
If $u$ does not lie on the geodesic through $w_{1}$ and $w_{2}$ and not at
distance $\pi /2$ from $w^{\ast }$ strict inequality holds, and since $K_{1}$
has nonempty interior strict inequality holds on an open subset of $K_{1}$.
If $g$ is increasing and convex in $\left[ 0,\pi /2\right] $ then dividing
by 2, applying $g$ and integrating over $K_{1}$ we get%
\[
\psi \left( w^{\ast }\right) <\left( 1/2\right) \left( \psi \left(
w_{1}\right) +\psi \left( w_{2}\right) \right) . 
\]%
It follows that there can be at most one point in $\hat{K}_{1}$ where the
gradient of $\psi $ vanishes, and this point, if it exists, must correspond
to a local minimum. By (ii) this is the unique global minimum and the only
local minimum outside $-\hat{K}$, which establishes (iii).

(iv) If $x_{1},x_{2}\in -\hat{K}_{1}$ and $x^{\ast }\in -\hat{K}_{1}$ is
their midpoint, then for $u\in K$ we obtain, using $\rho \left(
x_{i},u\right) =\pi -\rho \left( -x_{i},u\right) $ and a reasoning analogous
to the above,%
\[
\rho \left( u,w^{\ast }\right) \geq \left( 1/2\right) \left( \rho \left(
u,w_{1}\right) +\rho \left( u,w_{2}\right) \right) , 
\]%
the inequality being again strict on a set of positive measure and preserved
under application of a function $g$ which is increasing and concave in $%
\left[ \pi /2,\pi \right] $, so that%
\[
\psi \left( w^{\ast }\right) >\left( 1/2\right) \left( \psi \left(
w_{1}\right) +\psi \left( w_{2}\right) \right) . 
\]%
It again follows that there can be at most one point in $-\hat{K}_{1}$ where
the gradient of $\psi $ vanishes, and this point must now correspond to a
local maximum. We conclude that $\psi $ has a unique local minimum which
lies in $\hat{K}_{1}$.\bigskip
\end{proof}

\textbf{Remark.} An example of a function as in (iii) is $g\left( t\right)
=t^{2}$, in which case the minimizer is the spherical mass centroid
considered in \cite{Galperin 1993} and \cite{Buss 2001}. Examples of
functions as in (iv) are of course the identity function, in which case we
obtain the result stated in the introduction. We could also set $g\left(
t\right) =2\left( 1-\cos t\right) $, in which case the function reads%
\[
\psi \left( w\right) =\int_{K_{1}}\left\Vert w-y\right\Vert ^{2}d\sigma
\left( y\right) \text{.}
\]%
In this case uniqueness of the minimum can be established with much simpler
methods.\bigskip 

\textbf{Acknowledgement.} The author is grateful to Andreas Argyriou,
Massimiliano Pontil and Erhard Seiler for many encouraging discussions.

\end{document}